\documentclass[conference]{IEEEtran}
\IEEEoverridecommandlockouts

\usepackage{cite}
\usepackage{amsmath,amssymb,amsfonts}
\usepackage{algorithmic}
\usepackage{graphicx}
\usepackage{textcomp}
\usepackage{xcolor}

\def\BibTeX{{\rm B\kern-.05em{\sc i\kern-.025em b}\kern-.08em
    T\kern-.1667em\lower.7ex\hbox{E}\kern-.125emX}}

\usepackage{amssymb}
\usepackage{amsthm}
\usepackage{mathtools}
\usepackage[short]{optidef}
\usepackage{mathrsfs}
\usepackage{algorithm}
\usepackage{algorithmic}
\usepackage{multicol}
\usepackage{enumerate}
\usepackage{mathtools}

\newtheorem{theorem}{Theorem}

\newtheorem{assumption}{Assumption}
\theoremstyle{remark}
\newtheorem{remark}{Remark}

\newcommand{\norm}[1]{\left\lVert#1\right\rVert}
\newcommand{\mc}[1]{\mathcal{#1}}
\newcommand{\mb}[1]{\mathbb{#1}}
\newcommand{\ms}[1]{\mathscr{#1}}
\newcommand{\E}{\mathbb{E}}
\DeclarePairedDelimiter{\abs}{\lvert}{\rvert}

\newcounter{relctr} 
\everydisplay\expandafter{\the\everydisplay\setcounter{relctr}{0}} 

\newcommand\labelrel[2]{%
  \begingroup
    \refstepcounter{relctr}%
    \stackrel{\textnormal{(\alph{relctr})}}{\mathstrut{#1}}%
    \originallabel{#2}%
  \endgroup
}
\AtBeginDocument{\let\originallabel\label}

\makeatletter
\newcommand{\leqnomode}{\tagsleft@true\let\veqno\@@leqno}
\makeatother

\begin{document}

\title{Signal attenuation enables scalable decentralized multi-agent reinforcement learning over networks\\
}

\author{\IEEEauthorblockN{Wesley A. Suttle}
\IEEEauthorblockA{\textit{Army Research Laboratory} \\
\textit{U.S. Army DEVCOM}\\
Adelphi, MD USA \\
wesley.a.suttle.ctr@army.mil}
\and
\IEEEauthorblockN{Vipul K. Sharma}
\IEEEauthorblockA{\textit{Industrial Engineering Dept.} \\
\textit{Purdue University}\\
West Lafayette, IN USA \\
sharm697@purdue.edu}
\and
\IEEEauthorblockN{Brian M. Sadler}
\IEEEauthorblockA{\textit{Oden Institute} \\
\textit{University of Texas, Austin}\\
Austin, TX USA \\
brian.sadler@ieee.org}
}

\maketitle

\begin{abstract}
Multi-agent reinforcement learning (MARL) methods typically require that agents enjoy global state observability, preventing development of decentralized algorithms and limiting scalability. Recent work has shown that, under assumptions on decaying inter-agent influence, global observability can be replaced by local neighborhood observability at each agent, enabling decentralization and scalability. Real-world applications enjoying such decay properties remain underexplored, however, despite the fact that signal power decay, or signal attenuation, due to path loss is an intrinsic feature of many problems in wireless communications and radar networks. In this paper, we show that signal attenuation enables decentralization in MARL by considering the illustrative special case of performing power allocation for target detection in a radar network. To achieve this, we propose two new constrained multi-agent Markov decision process formulations of this power allocation problem, derive local neighborhood approximations for global value function and policy gradient estimates and establish corresponding error bounds, and develop decentralized saddle point policy gradient algorithms for solving the proposed problems. Our approach, though oriented towards the specific radar network problem we consider, provides a useful model for extensions to additional problems in wireless communications and radar networks.
\end{abstract}


\section{Introduction} \label{sec:introduction}

Multi-agent reinforcement learning (MARL) has seen immense attention in recent years, from both theoretical \cite{zhang2018fully, suttle2020multi, kosaraju2021reinforcement, zhang2021multi} and experimental \cite{hernandez2019survey, gronauer2022multi, zhu2024survey} perspectives. Due to limitations of the underlying multi-agent Markov decision process (MDP) model, however, standard methods for MARL in networked systems require global state observability at each agent \cite{zhang2021multi}. This inhibits the development of truly decentralized MARL methods where each agent only needs access to information from its local neighborhood, thereby limiting the scalability of such methods and preventing their application to realistic problems. Fortunately, recent works \cite{qu2020scalable, qu2022scalable, zhang2024scalable, shibl2025scalable} on scalable, decentralized MARL have shown that, for problems where inter-agent influence decays sufficiently quickly as the distance between agents increases over their communication network, use of only local neighborhood information at each agent suffices to approximately solve the global problem. This enables the development of truly decentralized methods that scale well as the number of agents increases. However, despite the theoretical advantages of these methods, real-world applications where the prerequisite decay properties hold remain largely unexplored.  

Signal power decay, or signal attenuation, due to path loss is a well-known property of wireless communications \cite{goldsmith2005wireless} and radar systems \cite{richards2005fundamentals}. In problems where multiple agents are widely dispersed over a geographic region, such as radar networks \cite{haykin2006cognitive}, path loss naturally leads to decay of inter-agent influence as distance between agents increases.
For example, when the performance metric at each agent is a function of the power of received signals, such as the signal-to-interference-plus-noise ratio (SINR), performance measurements at a given agent are largely decoupled from behavior of other agents that are sufficiently far away over the network.
Due to this inherent property of these systems, wireless communications and radar networks provide promising candidates for real-world application of scalable, decentralized MARL methods like \cite{qu2020scalable, qu2022scalable, zhang2024scalable, shibl2025scalable} that rely on such decay properties for success.

In this paper, we examine the implications of signal attenuation for the development of scalable, decentralized MARL approaches to the specific problem of performing power allocation for target detection in a radar network.
Radar networks are attractive for performing target detection and tracking due to advantages arising from their spatial dispersion and potential signal variety \cite{deligiannis2017bayesian}. When determining power allocations in radar networks, maximizing power leads to improved signal strength, but this conflicts with the need to achieve low probability of intercept (LPI) and abide by resource constraints \cite{shi2017power}. Existing methods for LPI power allocation for target detection in radar networks are centralized in that they require global observability and global coordination between radars \cite{shi2017power, deligiannis2017bayesian, snow2023identifying}, rendering them impractical in large networks for similar reasons to the MARL methods discussed above.

In this work, we propose a MARL approach for performing decentralized power allocation for LPI target detection in radar networks that mitigates these drawbacks. This is achieved by leveraging the signal attenuation inherent in radar networks to replace global observability and coordination with local observability and coordination among neighboring radars. Specifically, our contributions are as follows: (i) we propose two new constrained multi-agent MDP formulations of the problem of power allocation for target detection in radar networks; (ii) we leverage signal attenuation properties inherent in our setting to derive local approximations of the policy gradient expressions used in our algorithms and rigorously establish error bounds on these approximations; (iii) we propose novel decentralized, policy gradient ascent-descent algorithms for approximately solving the proposed problems. Though we focus on radar networks in this work, our approach can likely be extended to a broad range of applications in wireless communication and radar networks.
\section{Problem Formulation} \label{sec:formulation}
In this section, we first describe our system model for a
radar network with widely separated radars, a flying target, and extended clutter. To enable LPI target detection and tracking in this setting, we subsequently propose two different problem formulations: (i) maximizing sum-of-SINRs subject to regional power constraints; (ii) minimizing power consumption subject to a minimal SINR threshold and regional power constraints.

\subsection{Radar network model} \label{subsec:radar_network_model}
\begin{figure*}[htp]
    \centering
    \includegraphics[width=0.85\linewidth]{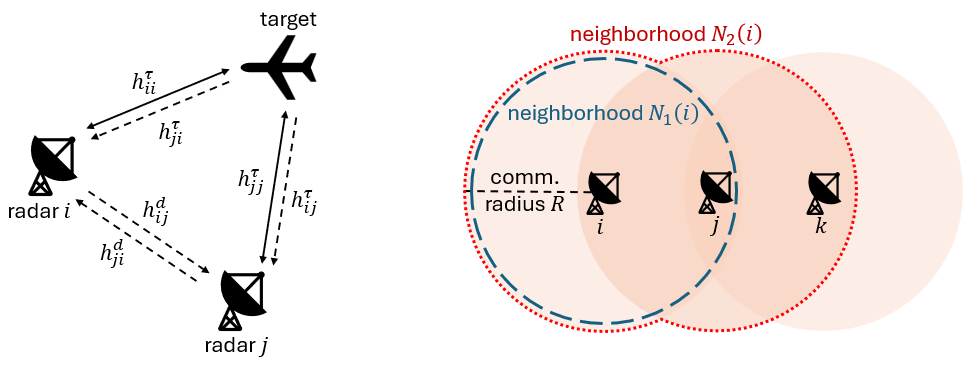}
    \caption{Illustrations of radar network system model of Section \ref{subsec:radar_network_model} (left) and communication neighborhoods of Section \ref{subsec:cmamdp} (right).}
    \label{fig:radar_network}
\end{figure*}
We consider a radar network composed of a set $\mc{N} = \{ 1, \ldots, n \}$ networked radars. In the presence of a target, the received signal for radar $i \in \mc{N}$ is given by \cite{deligiannis2017bayesian, shi2017power}
\begin{equation}
    x^i = \alpha^i \sqrt{a^i} y^i + \sum_{i \in \mc{N} \setminus \{ i \}} \beta^{ji} \sqrt{a^j} y^j + \omega^i,
\end{equation}
where $y^i = \psi^i z^i$ describes the transmitted signal from radar $i$ and $z^i = \left[ 1 \quad e^{j 2 \pi f_{D, i}} \quad \ldots \quad e^{j 2 \pi (N - 1) f_{D, i}} \right]^T$ is the Doppler steering vector of radar $i$ associated with the desired target, $f_{D, i}$ denotes the normalized Doppler shift as seen by radar $i$, $N$ is the number of received pulses at each timestep, and $\psi^i$ denotes the predesigned waveform transmitted from radar $i$. Furthermore, the parameter $\alpha^i$ denotes the desired channel gain in the target direction, $a^i$ denotes the transmission power of radar $i$, $\beta^{ji}$ describes the cross-channel gain between radars $i$ and $j$, and $\omega^i$ denotes zero-mean white Gaussian noise.

We assume that $\alpha^i \sim CN(0, h^{\tau}_{ii})$, $\beta^{ij} \sim CN(0, c_{ij} (h^{\tau}_{ij} + h^d_{ij}))$, and $\omega^i \sim CN(0, (\sigma^i)^2)$ for an $i$-dependent $\sigma^i > 0$, where $c_{ij} h^{\tau}_{ij}$ represents the variance of the channel gain for the radar $i$-target-radar $j$ path, $c_{ij} h^d_{ij}$ represents the variance of the channel gain for the direct radar $i$-radar $j$ path, and $c_{ij}$ denotes the cross-correlation coefficient between the $i$th and $j$th radars. Note that $c_{ii} = 1$, so the variance of the channel gain for the radar $i$-target-radar $i$ path is simply $h^{\tau}_{ii}$. The variances are given by the radar range equations \cite{richards2005fundamentals}
\begin{multicols}{2}
    \noindent
    \begin{equation}
        h^{\tau}_{ij} = \frac{ G_t G_r \sigma^{RCS}_{ij} \lambda^2 }{ (4 \pi)^3 R_i^2 R_j^2 }, \label{eqn:radar_range_1}
    \end{equation}
    \begin{equation}
        h^d_{ij} = \frac{ G_t' G_r' \lambda^2 }{ (4 \pi)^2 d_{ij}^2 }, \label{eqn:radar_range_2}
    \end{equation}
\end{multicols}
\noindent
where $G_t$ and $G_r$ are the radar main-lobe transmitting and receiving gains, $G_t'$ and $G_r'$ are the radar side-lobe transmitting and receiving gains, $\sigma_{ij}^{RCS}$ is the radar cross section (RCS) of the target between the $i$th and $j$th radars, $\lambda$ denotes the wavelength, $R_i$ denotes the distance between radar $i$ and the target, and $d_{ij}$ denotes the distance between radars $i$ and $j$. The signal-to-interference-plus-noise ratio (SINR) of the signal received at radar $i$ is then given by
\begin{equation}
    SINR_i = \frac{ h^{\tau}_{ii} a^i }{ (\sigma^i)^2 + \sum_{j \in \mc{N} \setminus \{i\}} c_{ji} \left( h^d_{ji} + h^{\tau}_{ji} \right) a^j }, \label{eqn:SINR_static}
\end{equation}
where $\sigma^i$ corresponds to the noise received at radar $i$. See Figure \ref{fig:radar_network} for an illustration of the radar network system model.
\subsection{Constrained multi-agent MDP formulations} \label{subsec:cmamdp}

To capture the underlying system model described above, we use a constrained multi-agent Markov decision process (CMAMDP) $(\mc{S}, \mc{A}, p, \mc{N}, \{ r^i \}_{i \in \mc{N}}, \{ c^i \}_{i \in \mc{N}} )$, defined below.
Let the set of radars $\mc{N} = \{ 1, \ldots, n \}$ correspond to the $n$ radars.
Let $\mc{S} = \mc{S}^1 \times \ldots \times \mc{S}^n$ denote the joint state space, where the $i$th component $\mc{S}^i \subset \mb{R}^k \times \mb{R}^m$ corresponds to set of possible locations and movements of the target and radar $i$, i.e., $s^i = (s^i_{target}, s^i_{radar})$ with $s^i_{target} \in \mb{R}^k$ and $s^i_{radar} \in \mb{R}^m$. When the target is moving in $\mb{R}^3$ and the radars are moving in $\mb{R}^2$, for example, we may take $k = 9$ and $m = 6$ and let $s^i_{target}$ and $s^i_{radar}$ correspond to the positions, velocities, and accelerations of the target and radar, respectively. Notice that, for all $i, j \in \mc{N}$, we have $s^i_{target} = s^j_{target}$.
Given the joint state $s_t \in \mc{S}$ at time $t$, we assume that the Doppler steering vectors, predesigned waveforms, and channel gain variances corresponding to radar $i$ are provided and that the goal of radar $i$ is to determine an appropriate transmission power level to supply given its local state information $s^i_t$. To capture this, let $\mc{A} = \mc{A}^1 \times \ldots \times \mc{A}^n$, where $\mc{A}^i = [0, a^{max}]$ denotes the set of possible power allocations at radar $i$, for a predetermined, finite, maximum power allocation $a^{max} > 0$ over all radars.

Let the transition dynamics $p : \mc{S} \times \mc{A} \rightarrow \mc{S}$ capture the movement of the target and radars over time. We assume in this paper that the movements of the target and radars, and therefore $p$, are independent of the transmission power allocations applied at the radars. Given joint state $s \in \mc{S}$ and joint power allocation $a \in \mc{S}$, let $r^i(s, a) = SINR_i(s, a)$ denote the SINR received at radar $i$ obtained from equation \eqref{eqn:SINR_static} by substituting the channel gains, cross-correlation coefficients, and noise corresponding to $s$ and applying allocation $a$, i.e.,
\small
\begin{equation}
    r^i(s, a) = \frac{ h^{\tau}_{ii}(s) a^i }{ (\sigma^i(s, a))^2 + \sum_{j \in \mc{N} \setminus \{i\}} c_{ji}(s) \left( h^d_{ji}(s) + h^{\tau}_{ji}(s) \right) a^j }, \label{eqn:SINR_dynamic}
\end{equation}
\normalsize
where we abuse notation to let $h^{\tau}_{ii}(s), h^{\tau}_{ji}(s), h^d_{ji}(s), c_{ji}(s), \sigma^i(s, a)$, and $\sigma^i_{\kappa}(s, a)$ correspond to the expressions appearing in \eqref{eqn:SINR_static} when the system is in state $s$ and joint action $a$ is selected.
Similarly, define the local neighborhood reward by
\footnotesize
\begin{equation}
    r^i_{\kappa}(s, a) = \frac{ h^{\tau}_{ii}(s) a^i }{ (\sigma^i_{\kappa}(s, a))^2 + \sum_{j \in \mc{N}_{\kappa}(i) \setminus \{ i \}} c_{ji}(s) \left( h^d_{ji}(s) + h^{\tau}_{ji}(s) \right) a^j }, \label{eqn:SINR_dynamic_kappa}
\end{equation}
\normalsize
which captures the SINR received at radar $i$ originating within neighborhood $\mc{N}_{\kappa}(i)$, where $\sigma_{\kappa}^i(s, a)$ denotes the noise originating within $\mc{N}_{\kappa}(i)$. The expression in \eqref{eqn:SINR_dynamic_kappa} will be crucial in the theoretical results of Section \ref{sec:decentralization}.
Finally, let the cost $c^i(s, a)$ denote the cost to radar $i$ of applying power level $a^i$. We might simply take $c^i(s, a) = a^i$, for example, but our approach accommodates general cost structures.

Let a fixed communication radius $R > 0$ between the radars be given, and denote the undirected communication graph between radars in state $s \in \mc{S}$ by $\mc{G}(s) = (\mc{N}, \mc{E}(s))$, where $\mc{E}(s) = \{ (i, j) \ | \ d_{ij}(s) \leq R \}$ and $d_{ij}(s)$ denotes the Euclidean distance between radars $i$ and $j$ when the system is in state $s$. We note that $R$ defines a user-specified communications neighborhood, but that it does not necessarily represent a hard limit on communications within the network. We assume in this paper that the edge set $\mc{E} = \mc{E}(s)$ remains constant, i.e., that any movement of the radars leaves the topology of the communication network unchanged, and will henceforth suppress the dependence on $s$ and simply write $\mc{G} = (\mc{N}, \mc{E})$. For a given positive integer $\kappa \in \mb{N}^+$, let $\mc{N}_{\kappa}(i)$ denote the $\kappa$-hop neighborhood of radar $i$ with respect to $\mc{G}$. Note that, since $\mc{G}$ is undirected, the $\kappa$-hop neighborhood relation is symmetric, i.e., $i \in \mc{N}_{\kappa}(j)$ if and only if $j \in \mc{N}_{\kappa}(i)$. Finally, let $\mc{N}_{\kappa}^{-1}(i) = \mc{N} \setminus \mc{N}_{\kappa}(i)$ denote the set of all radars outside $i$'s $\kappa$-hop neighborhood. See Figure \ref{fig:radar_network} for an illustration of the communication neighborhoods induced by a specified communication radius $R$ on a simple example.
%
%
%

Let $\kappa \in \mb{N}^+$ be fixed, and assume each radar $i$ has access to the state information $s^{\mc{N}_{\kappa}(i)} \in \mc{S}^{\mc{N}_{\kappa}(i)} = \{ s^j \ | \ j \in \mc{N}_{\kappa}(i) \}$. To each $i$, let there be associated a parameterized policy class $\{ \pi^i_{\theta_i} : \mc{S}^{\mc{N}_{\kappa}(i)} \rightarrow \Delta(\mc{A}^i) \}_{\theta^i \in \Theta^i}$, where $\Theta^i \subseteq \mathbb{R}^d$ is the set of permissible policy parameters, for some positive integer $d$. Denote the induced joint policy by $\pi_{\theta}(a | s) = \prod_{i \in \mc{N}} \pi^i_{\theta^i}(a^i | s^{\mc{N}_{\kappa}(i)})$, where $\theta = \left[ (\theta^1)^T \ldots (\theta^n)^T \right]^T \in \Theta = \Theta^1 \times \ldots \times \Theta^n$ is the stacked vector of each radar's policy parameters. Given policy $\pi_{\theta}$, under this formulation,
\small
\begin{align}
    J_{r^i}(\theta) &= \lim_{T \rightarrow \infty} \frac{1}{T} \mb{E}_{\pi_{\theta}} \Big[ \sum_{t=0}^{T-1} r^i(s_t, a_t) \Big], \label{eqn:radar_average_reward} \\
    \quad J_{c^i}(\theta) &= \lim_{T \rightarrow \infty} \frac{1}{T} \mb{E}_{\pi_{\theta}} \Big[ \sum_{t=0}^{T-1} c^i(s_t, a_t) \Big] \label{eqn:radar_average_cost}
\end{align}
\normalsize
capture the average expected SINR achieved and average expected cost incurred, respectively, at radar $i$.
%

\subsubsection{Sum-of-SINRs Maximization} Fix $\kappa > 0$. The sum-of-SINRs maximization problem is as follows.
\begin{maxi*}
  {\theta \in \Theta}{ \sum_{i \in \mathcal{N}} J_{r^i}(\theta) }{}{}
    \addConstraint{ \sum_{j \in \mc{N}_{\kappa}(i) } J_{c^j}(\theta) }{ \leq u^i, }{ \hspace{3mm} \forall i \in \mc{N}, } \tag{$P^{max}_{\kappa}$} \label{opt:Q_kappa}
\end{maxi*}
The objective of this problem is to maximize the sum over the expected average SINRs over all radars while simultaneously ensuring that expected average cost does not exceed ``regional'' upper bounds at each radar.

\subsubsection{Power Minimization with SINR Threshold} Fix $\kappa > 0$ and let $\gamma_{min} > 0$ denote the minimum SINR value allowable at each radar, where we assume that $\gamma_{min}$ is chosen such that the thresholds at each radar in problem \eqref{opt:R_kappa} below are achievable. The problem of minimizing overall power consumption while respecting the SINR threshold is
\begin{mini*}
  {\theta \in \Theta}{ \sum_{i \in \mathcal{N}} J_{c^i}(\theta) }{}{}
    \addConstraint{ J_{r^i}(\theta) }{ \geq \gamma_{\min}, }{ \hspace{3mm} \forall i \in \mc{N}} \tag{$P^{min}_{\kappa}$} \label{opt:R_kappa}
    \addConstraint{ \sum_{j \in \mc{N}_{\kappa}(i) } J_{c^j}(\theta) }{ \leq u^i, }{ \hspace{3mm} \forall i \in \mc{N}. }
\end{mini*}
The objective of this problem is to minimize the expected average cost over all radars while simultaneously ensuring that both (a) each radar achieves a minimal SINR value, on average, and (b) expected average cost does not exceed ``regional'' upper bounds at each radar.
\section{Signal attenuation enables decentralization} \label{sec:decentralization}

In this section, we examine how signal decay inherent in the radar problem under consideration leads to natural decentralization in MARL solution approaches. We first recall the MARL global observability issue that prevents true decentralization in the general setting, then establish formal properties of our problem that lead to natural decentralization. We leverage the properties established in this section to develop decentralized MARL algorithms for solving problems \eqref{opt:Q_kappa} and \eqref{opt:R_kappa} in the next section.

We first review a few key concepts from RL. Let $r(s, a) = \sum_{i \in \mc{N}} r^i(s, a)$ denote the global reward and $c(s, a) = \sum_{i \in \mc{N}} c^i(s, a)$ the global cost. Define the global averages
\begin{equation}
    J_r(\theta) = \sum_{i \in \mc{N}} J_{r^i}(\theta), \qquad J_c(\theta) = \sum_{i \in \mc{N}} J_{c^i}(\theta).
\end{equation}
Fixing $\theta \in \Theta$, the action-value function corresponding to $\pi_{\theta}$ and the reward or cost function $f \in \{ r, c \}$ is given by
\small
\begin{align}
    Q^f_{\theta}&(s, a) = \mb{E}_{\pi_{\theta}} \left[ \sum_{t=0}^\infty f(s_t, a_t) - J_f(\theta) \ | \ s_0 = s, a_0 = a \right] \\
    &= \sum_{i \in \mc{N}} \mb{E}_{\pi_{\theta}} \left[ \sum_{t=0}^{\infty} f^i(s, a) - J_{f^i}(\theta) \ | \ s_0 = s, a_0 = a \right] \\
    &= \sum_{i \in \mc{N}} Q^{f^i}_{\theta}(s, a),
\end{align}
\normalsize
where the $Q$ function $Q^{f^i}_{\theta}(s, a)$ corresponding to $f^i$ is
\small
\begin{equation}
    Q^{f^i}_{\theta}(s, a) = \mb{E}_{\pi_{\theta}} \left[ \sum_{t=0}^{\infty} f^i(s, a) - J_{f^i}(\theta) \ | \ s_0 = s, a_0 = a \right],
\end{equation}
\normalsize
for the $i$th reward or cost $f^i \in \{ r^i, c^i \}$.

\subsection{MARL Global Observability Issue}
Ideally, we would like each agent $i$ to be able to compute its contribution $\nabla_{\theta^i} J_f(\theta)$ to the overall gradient $\nabla_{\theta} J_f(\theta)$ using only information from its local neighborhood $\mc{N}_{\kappa}(i)$, as this would make the development of decentralized algorithms easier. However, applying the classic policy gradient theorem \cite{sutton1999policy} to differentiate with respect to $\theta^i$ gives
\small
\begin{align}
    \nabla_{\theta^i} J_f(\theta) &= \mb{E}_{\pi_{\theta}} \left[ Q^f_{\theta}(s, a) \nabla_{\theta^i} \log \pi^i_{\theta^i}(a^i | s^i) \right] \\
    &= \mb{E}_{\pi_{\theta}} \left[ \sum_{j \in \mc{N}} Q^{f^j}_{\theta}(s, a) \nabla_{\theta^i} \log \pi^i_{\theta^i}(a^i | s^i) \right]. \label{eqn:global_obs}
\end{align}
\normalsize
From equation \eqref{eqn:global_obs}, global information is clearly required to estimate $\nabla_{\theta^i} J_f(\theta)$, since each of the $Q^{f^j}_{\theta}(s, a)$ requires access to the global state $s$ and joint action $a$, and agent $i$ needs access to the $Q$ functions of all other agents $j \in \mc{N}$. This is the crux of the global observability issue in policy gradient methods for MARL. To address this issue, we next identify properties of the radar problem we consider under which the term \small $\sum_{j \in \mc{N}} Q^{f^j}_{\theta}(s, a)$ \normalsize can be replaced by an approximation depending only on the information available to agent $i$ within its local neighborhood, $\mc{N}_{\kappa}(i)$.

\subsection{Signal Attenuation and Decentralization}

We start this section by stating several assumptions that will be needed in the subsequent analysis. Assumption \ref{assum:network_coverage} provides a mechanism for ensuring that the radar network provides adequate coverage of the region under consideration through appropriate choice of $g$. Assumption \ref{assum:pairwise_cost_independence} stipulates that the local costs at each radar are independent of the costs at all other radars. This is reasonable when costs depend only on local power consumption, for example. Assumption \ref{assum:regularity_conditions} provides minimal conditions ensuring that the radar range equations \eqref{eqn:radar_range_1}-\eqref{eqn:radar_range_2} lead to well-conditioned local rewards \eqref{eqn:SINR_dynamic}. Assumptions \ref{assum:uniform_ergodicity} and \ref{assum:lipschitz_score_functions} are commonly used in the RL \cite{wu2020finite, chen2023finite, suttle2023beyond} and MARL \cite{qu2020scalable, qu2022scalable} literatures to enable analysis of policy gradient-based methods. Assumption \ref{assum:bdd_interagent_gradients} bounds the effect that changes in a given radar's policy parameters can have on rewards received by radars outside its $\kappa$-hop neighborhood, which is reasonable when rewards decay as distance between agents increases. We suspect that Assumption \ref{assum:bdd_interagent_gradients} can be proven to hold for our setting, but leave this to future work.
\begin{assumption}[Network Coverage] \label{assum:network_coverage}
    There exists a function $g : \mathbb{N}^+ \times \mathbb{R}^+ \rightarrow [1, \infty)$, strictly increasing in both entries, such that, for all $i \in \mc{N}$, all $\kappa \in \mathbb{N}^+$, and any $j \in \mc{N}^{-1}_{\kappa}(i)$, we have $d_{ij}(s) \geq g(\kappa, R)$, for all $s \in \mc{S}$.
\end{assumption}
\begin{remark}
    When $\kappa = 1$, taking $g(\kappa, R) = R$ is a natural choice in Assumption \ref{assum:network_coverage}, since, for any $j \in \mc{N}^{-1}_{1}(i)$, we have $d_{ij}(s) \geq g(1, R) = R$ by definition of the neighbor relation in $\mc{G}$ (see Section \ref{subsec:cmamdp}). More generally, letting $g(\kappa, R) = \kappa R$ satisfies Assumption \ref{assum:network_coverage} for a large class of network topologies, such as when the radars are static and arranged in a grid topology with grid cell side lengths $R$.
    %
\end{remark}
%
%
%
\begin{assumption}[Pairwise Cost Independence] \label{assum:pairwise_cost_independence}
    For each $i \in \mc{N}$, we have the following two conditions: (i) $\nabla_{\theta^i} J_{c^j}(\theta) = 0$, for all $j \in \mc{N} \setminus \{ i \}$; (ii) $Q^{c^i}_{\theta}(s, a) = Q^{c^i}_{\theta}(s^i, a^i)$, i.e., the value of $Q^{c^i}$ depends on purely local state and action information.
\end{assumption}
\begin{assumption}[Regularity Conditions] \label{assum:regularity_conditions}
    For all $i, j \in \mc{N}$, we have $\inf_{s \in \mc{S}} R_i(s) \geq 1$, $\bar{\sigma}^{RCS} = \sup_{s \in \mc{S}} \sigma_{ij}^{RCS}(s) < \infty$ and $\bar{\sigma}^{RCS} > 0$, $\bar{c} = \sup_{s \in \mc{S}} c_{ji}(s) < \infty$ and $\bar{c} > 0$, $\underset{\bar{}}{\sigma} = \min \{ \inf_{s, a} \sigma^i(s, a), \inf_{s, a} \sigma_{\kappa}^i(s, a) \} > 0$, and $\lambda, G_t, G_r, G_t', G_r' > 0$.
\end{assumption}
\begin{assumption}[Uniform Ergodicity] \label{assum:uniform_ergodicity}
    There exist $\rho \in (0, 1)$ and $m \in \mathbb{R}^+$ such that every joint policy $\pi_{\theta}$ satisfies $d_{TV}(d_{\pi_{\theta}}^t( \cdot | s_0 ) || d_{\theta}(\cdot)) \leq m \rho^t$, for any $s_0 \in \mc{S}$ and for all $t \geq 0$, where $d_{TV}(q(\cdot) || q'(\cdot)) = \sup_{A} \abs{ \int_{A} q(x) \ dx - \int_{A} q'(x) \ dx }$ denotes the total variation distance between densities $q, q'$, and $d_{\pi_{\theta}}^t(\cdot | s_0)$ denotes the $t$-step state occupancy measure induced by $\pi_{\theta}$ over $\mc{S}$ given start state $s_0$.
\end{assumption}
\begin{assumption}[Lipschitz Score Functions] \label{assum:lipschitz_score_functions}
    For each $i \in \mc{N}$, there exists $L^i > 0$ such that $\norm{ \nabla_{\theta^i} \log \pi^i_{\theta^i}(a^i | s^i) } \leq L^i$, for all $s^i \in \mc{S}^i, a^i \in \mc{A}^i$.
\end{assumption}
\begin{assumption}[Bounded Inter-agent Gradients] \label{assum:bdd_interagent_gradients}
    There exists $\varepsilon_{\kappa} > 0$ such that, for each $i \in \mc{N}$ and all $j \in \mc{N}_{\kappa}^{-1}(i)$, we have $\norm{ \nabla_{\theta^i} J_{r^j}(\theta)} \leq \varepsilon_{\kappa}$, for all $\theta \in \Theta$.
\end{assumption}

We now establish properties of the radar problem enabling decentralized solution of the CMAMDP problem.
Due to the form of the costs and rewards coupled with properties of the SINR \eqref{eqn:SINR_static} and radar range equations \eqref{eqn:radar_range_1} and \eqref{eqn:radar_range_2}, the CMAMDP enjoys the following property.
\begin{theorem} \label{thm:radar_exp_decay}
Let Assumptions \ref{assum:network_coverage}, \ref{assum:pairwise_cost_independence}, \ref{assum:regularity_conditions}, and \ref{assum:uniform_ergodicity} hold. For any $\theta \in \Theta$, $i \in \mc{N}$, $s^{\mc{N}_{\kappa}(i)} \in \mc{S}^{\mc{N}_{\kappa}(i)}$, $a^{\mc{N}_{\kappa}(i)} \in \mc{A}^{\mc{N}_{\kappa}(i)}$, and $f^i \in \{ r^i, c^i \}$, we have
\small
\begin{align*}
    \Big| &Q_{\theta}^{f^i} \left( ( s^{\mc{N}_{\kappa}(i)}, s^{\mc{N}^{-1}_{\kappa}(i)} ), (a^{\mc{N}_{\kappa}(i)}, a^{\mc{N}^{-1}_{\kappa}(i)} ) \right) \\
    &- Q_{\theta}^{f^i} \left( ( s^{\mc{N}_{\kappa}(i)}, \bar{s}^{\mc{N}^{-1}_{\kappa}(i)} ), (a^{\mc{N}_{\kappa}(i)}, \bar{a}^{\mc{N}^{-1}_{\kappa}(i)} ) \right) \Big| \leq \frac{M | \mc{N}_{\kappa}^{-1}(i) |}{g^2(\kappa, R)},
\end{align*}
\normalsize
for all $s^{\mc{N}^{-1}_{\kappa}(i)}, \bar{s}^{\mc{N}^{-1}_{\kappa}(i)} \in \mc{S}^{\mc{N}^{-1}_{\kappa}(i)}$ and all $a^{\mc{N}^{-1}_{\kappa}(i)}, \bar{a}^{\mc{N}^{-1}_{\kappa}(i)} \in \mc{A}^{\mc{N}^{-1}_{\kappa}(i)}$, where $M > 0$ is a constant depending only on the quantities in Assumptions \ref{assum:regularity_conditions} and \ref{assum:uniform_ergodicity}.
\end{theorem}
\begin{proof}
    Assumption \ref{assum:pairwise_cost_independence} renders the $f^i = c^i$ case trivial, so we give the proof only for the $f^i = r^i$ case. Fix $\theta \in \Theta, i \in \mc{N},$ and $\kappa \in \mathbb{N}^+$. Let $s^{\mc{N}_{\kappa}(i)} \in \mc{S}^{\mc{N}_{\kappa}(i)}$, $a^{\mc{N}_{\kappa}(i)} \in \mc{A}^{\mc{N}_{\kappa}(i)}$, $s^{\mc{N}^{-1}_{\kappa}(i)}, \bar{s}^{\mc{N}^{-1}_{\kappa}(i)} \in \mc{S}^{\mc{N}^{-1}_{\kappa}(i)}$, and $a^{\mc{N}^{-1}_{\kappa}(i)}, \bar{a}^{\mc{N}^{-1}_{\kappa}(i)} \in \mc{A}^{\mc{N}^{-1}_{\kappa}(i)}$. Define $s = (s^{\mc{N}_{\kappa}(i)}, s^{\mc{N}_{\kappa}^{-1}(i)}), \bar{s} = (s^{\mc{N}_{\kappa}(i)}, \bar{s}^{\mc{N}_{\kappa}^{-1}(i)}), a = (a^{\mc{N}_{\kappa}(i)}, a^{\mc{N}_{\kappa}^{-1}(i)})$, and $\bar{a} = (a^{\mc{N}_{\kappa}(i)}, \bar{a}^{\mc{N}_{\kappa}^{-1}(i)})$.

    By the definition of $Q^{r^i}_{\theta}$, we have
    \begin{align}
        &\abs*{Q^{r^i}_{\theta}(s, a) - Q^{r^i}_{\theta}(\bar{s}, \bar{a}) } \\
        &= \big| \E_{\pi_{\theta}} \big[ \sum_{t=0}^{\infty} r^i(s_t, a_t) - J_{r^i}(\theta) \ | \ s_0 = s, a_0 = a \big] \nonumber \\
        &\qquad - \E_{\pi_{\theta}} \big[ \sum_{t=0}^{\infty} r^i(s_t, a_t) - J_{r^i}(\theta) \ | \ s_0 = \bar{s}, a_0 = \bar{a} \big] \big| \label{eqn:thm1:1} \\
        &= \big| \E_{\pi_{\theta}} \big[ \sum_{t=0}^{\infty} r^i(s_t, a_t) \ | \ s_0 = s, a_0 = a \big] \nonumber \\
        &\qquad - \E_{\pi_{\theta}} \big[ \sum_{t=0}^{\infty} r^i(s_t, a_t) \ | \ s_0 = \bar{s}, a_0 = \bar{a} \big] \big|, \label{eqn:thm1:2}
    \end{align}
    where the last equality holds by the fact that $\E_{\pi_{\theta}} \left[ J_{r^i}(\theta) \ | \ s_0 = s, a_0 = a \right] = \E_{\pi_{\theta}} \left[ J_{r^i}(\theta) \ | \ s_0 = \bar{s}, a_0 = \bar{a} \right]$, since Assumption \ref{assum:uniform_ergodicity} implies ergodicity of the Markov chain over $\mc{S}$ induced by $\pi_{\theta}$.
    Recalling \eqref{eqn:SINR_dynamic_kappa} and continuing from \eqref{eqn:thm1:2}, we have
    \small
    \begin{align}
        &\abs*{Q^{r^i}_{\theta}(s, a) - Q^{r^i}_{\theta}(\bar{s}, \bar{s}) } \\
        &= \Big| \sum_{t=0}^\infty \Big[ \E_{\pi_{\theta}} \big[ r^i(s_t, a_t) - r^i_{\kappa}(s_t, a_t) + r^i_{\kappa}(s_t, a_t) \ | \ s_0 = s, a_0 = a \big] \nonumber \\
        &\qquad - \E_{\pi_{\theta}} \big[ r^i(s_t, a_t) \ | \ s_0 = \bar{s}, a_0 = \bar{a} \big] \Big] \Big| \label{eqn:thm1:3} \\
        &= \Big| \sum_{t=0}^{\infty} \Big[ \E_{\pi_{\theta}} \big[ r^i(s_t, a_t) - r^i_{\kappa}(s_t, a_t) \ | \ s_0 = s, a_0 = a \big] \nonumber \\
        &\qquad - \E_{\pi_{\theta}} \big[ r^i(s_t, a_t) - r^i_{\kappa}(s_t, a_t) \ | \ s_0 = \bar{s}, a_0 = \bar{a} \big] \Big] \Big|, \label{eqn:thm1:4}
    \end{align}
    \normalsize
    where \eqref{eqn:thm1:4} follows by the independence of $a$ of the transition dynamics $p$ defined in Section \ref{subsec:cmamdp} and the fact that $r^i_{\kappa}(s, a)$ is independent of the value of $s^{\mc{N}_{\kappa}^{-1}(i)}, a^{\mc{N}_{\kappa}^{-1}(i)}$ by definition.

    Now notice that
    \small
    \begin{align}
        &\sum_{t=0}^\infty \E_{\pi_{\theta}} \big[ r^i(s_t, a_t) - r^i_{\kappa}(s_t, a_t) \ | \ s_0 = s, a_0 = a \big] \\
        &=\sum_{t=0}^\infty \sum_{s' \in \mc{S}} \sum_{a' \in \mc{A}} G^i_{\kappa}(s', a') \pi_{\theta}(a' | s') d_{\pi_{\theta}}^t(s' | s_0 = s, a_0 = a), \label{eqn:thm1:4:2}
    \end{align}
    \normalsize
    where $d_{\pi_{\theta}}^t$ denotes the $t$-step state distribution of Assumption \ref{assum:uniform_ergodicity} and $G^i_{\kappa}(s', a') = r^i(s', a') - r^i_{\kappa}(s', a')$, and that an analogous expression holds for $s_0 = \bar{s}, a_0 = \bar{a}$. Continuing from \eqref{eqn:thm1:4} and \eqref{eqn:thm1:4:2}, we have
    \small
    \begin{align}
        &\abs*{Q^{r^i}_{\theta}(s, a) - Q^{r^i}_{\theta}(\bar{s}, \bar{s}) } \\
        &= \Big| \sum_{t=0}^\infty \sum_{s' \in \mc{S}} \sum_{a' \in \mc{A}} G_{\kappa}^i(s', a') \pi_{\theta}(a' | s') \Big( d_{\pi_{\theta}}^t (s' | s_0 = s, a_0 = a) \nonumber \\
        &\hspace{4cm} - d_{\pi_{\theta}}^t (s' | s_0 = \bar{s}, a_0 = \bar{a}) \Big) \Big| \\
        &\labelrel{\leq}{(a)} \sum_{t=0}^\infty \sum_{s' \in \mc{S}} \sum_{a' \in \mc{A}} \abs*{ G^i_{\kappa}(s',a') } \cdot \big| d_{\pi_{\theta}}^t(s' | s_0 = s, a_0 = a) \nonumber \\
        &\hspace{4cm} - d_{\pi_{\theta}}^t(s' | s_0 = \bar{s}, a_0 = \bar{a}) \big| \\
        &\labelrel{\leq}{(b)} \sum_{t=0}^\infty \sum_{s' \in \mc{S}} \sum_{a' \in \mc{A}} \abs*{ G^i_{\kappa}(s',a') } \Big[ \abs*{d_{\pi_{\theta}}^t(s' | s_0 = s, a_0 = a) - d_{\pi_{\theta}}(s') } \nonumber \\
        &\hspace{2cm} + \abs*{ d_{\pi_{\theta}}^t(s' | s_0 = \bar{s}, a_0 = \bar{a}) - d_{\pi_{\theta}}(s') } \Big] \\
        &\labelrel{\leq}{(c)} \sup_{s', a'} \abs*{G^i_{\kappa}(s', a')} \sum_{t=0}^\infty \Big[ d_{TV}( d_{\pi_{\theta}}^t( \cdot | s_0 = s, a_0 = a) \ || \ d_{\pi_{\theta}}(\cdot) ) \nonumber \\
        &\hspace{2cm} + d_{TV}( d_{\pi_{\theta}}^t( \cdot | s_0 = \bar{s}, a_0 = \bar{a}) \ || \ d_{\pi_{\theta}}(\cdot) ) \Big] \\
        &\labelrel{\leq}{(d)} \sup_{s', a'} \abs*{ G^i_{\kappa}(s', a') } \sum_{t=0}^\infty 2 m \rho^t = \sup_{s', a'} \abs*{ G^i_{\kappa}(s', a') } \frac{2m}{1 - \rho}, \label{eqn:thm1:5}
    \end{align}
    \normalsize
    where inequality (a) follows by the triangle and Cauchy-Schwarz inequalities and the fact that $0 \leq \pi_{\theta}(a' | s') \leq 1$, inequality (b) follows by adding and subtracting $d_{\theta}(s')$ and applying the triangle inequality, (c) follows by taking suprema and the definition of total variation distance, and (d) follows by Assumption \ref{assum:uniform_ergodicity}.

    All that remains is to bound $\sup_{s', a'} \abs*{ G^i_{\kappa}(s', a') }$. Fix $s \in \mc{S}, a \in \mc{A}$. Recall that
    \small
    \begin{equation}
        r^i(s, a) = \frac{h^{\tau}_{ii}(s) a^i}{\sigma^2(s, a) + \alpha},
        \quad r^i_{\kappa}(s, a) = \frac{h^{\tau}_{ii}(s) a^i}{(\sigma_{\kappa}(s, a))^2 + \beta},
    \end{equation}
    \normalsize
    where
    \begin{align}
        \alpha &= \sum_{j \in \mc{N} \setminus \{ i \}} c_{ji}(s) \left[ h^d_{ji}(s) + h^{\tau}_{ji}(s) \right] a^j, \\
        \beta &= \sum_{j \in \mc{N}_{\kappa}(i) \setminus \{ i \}} c_{ji}(s) \left[ h^d_{ji}(s) + h^{\tau}_{ji}(s) \right] a^j.
    \end{align}
    In light of this, we can write
    \begin{align}
        \abs*{G^i_{\kappa}(s, a)} &= \abs*{ \frac{h^{\tau}_{ii}(s) a^i (\alpha - \beta)}{(\sigma^2(s, a) + \alpha)((\sigma^i_{\kappa}(s, a))^2 + \beta)}} \\
        &\labelrel{=}{(a)} \frac{h^{\tau}_{ii}(s) a^i \abs*{\alpha - \beta}}{\abs*{\sigma^2(s, a) + \alpha} \abs*{(\sigma^i_{\kappa}(s, a))^2 + \beta}} \\
        &\labelrel{\leq}{(b)} \frac{h^{\tau}_{ii}(s) a^i \abs*{\alpha - \beta}}{\underset{\bar{}}{\sigma}^4}, \label{eqn:thm1:6}
    \end{align}
    where (a) follows by nonnegativity of $h^{\tau}_{ii}(s)$ and $a^i$, and (b) follows by nonnegativity of $\alpha$ and $\beta$ and Assumption \ref{assum:regularity_conditions}. By \eqref{eqn:radar_range_1} and Assumption \ref{assum:regularity_conditions},
    \begin{align}
        h^{\tau}_{ii}(s) a^i = \frac{ G_t G_r \sigma^{RCS}_{ii}(s) \lambda^2 a^i }{ (4 \pi)^3 R_i^4(s) } \leq \frac{ G_t G_r \bar{\sigma}^{RCS} \lambda^2 a^{max} }{ (4 \pi)^3 },
    \end{align}
    whence it follows that
    \begin{align}
        \abs*{G^i_{\kappa}(s, a)} \leq \frac{G_t G_r \bar{\sigma}^{RCS} \lambda^2 a^{max} }{(4\pi)^3 \underset{\bar{}}{\sigma}^4 } \abs*{\alpha - \beta}. \label{eqn:thm1:7}
    \end{align}

    It remains to bound $\abs*{\alpha - \beta}$. Notice that
    \begin{align}
        \alpha - \beta &= \sum_{j \in \mc{N}_{\kappa}^{-1}(i)} c_{ji}(s) \left[ h^d_{ji}(s) + h^{\tau}_{ji}(s) \right] a^j \\
        &\leq \bar{c} a^{max} \sum_{j \in \mc{N}_{\kappa}^{-1}(i)} \left[ h^d_{ji}(s) + h^{\tau}_{ji}(s) \right] \label{eqn:thm1:8}
    \end{align}
    so we need only bound $h^d_{ji}(s)$ and $h^{\tau}_{ji}(s)$. For $h^d_{ji}(s)$, notice that, for $j \in \mc{N}_{\kappa}^{-1}(i)$, by \eqref{eqn:radar_range_2} and Assumption \ref{assum:network_coverage} we have
    \begin{align}
        h^d_{ji}(s) = \frac{G_t' G_r' \lambda^2}{(4 \pi)^2 (d_{ij}(s))^2} \leq \frac{G_t' G_r' \lambda^2}{(4 \pi)^2 g^2(\kappa, R)}. \label{eqn:thm1:9}
    \end{align}
    For $h^{\tau}_{ji}(s)$, we consider four possible cases that may arise depending on the physical configuration of the radar network. Fix $j \in \mc{N}_{\kappa}^{-1}(i)$.
    \\ \\
    \textbf{Case 1:} $R_j(s) \leq d_{ij}(s) \leq R_i(s)$. In this case, we know $R_i(s) R_j(s) \geq d_{ij}(s)$, whence
    \begin{align}
        h^{\tau}_{ji}(s) = \frac{G_t G_r \sigma^{RCS}(s) \lambda^2}{(4 \pi)^3 R^2_i(s) R^2_j(s)} \leq \frac{G_t G_r \bar{\sigma}^{RCS} \lambda^2}{(4\pi)^3 g^2(\kappa, R)},
    \end{align}
    where the inequality holds by Assumptions \ref{assum:network_coverage} and \ref{assum:regularity_conditions}.
    \\ \\
    \textbf{Case 2:} $R_i(s) \leq d_{ij}(s) \leq R_j(s)$. This case follows by reasoning identical to Case 1.
    \\ \\
    \textbf{Case 3:} $d_{ij}(s) \leq \min \{ R_i(s), R_j(s) \}$. In this case, $R_i(s) R_j(s) \geq d^2_{ij}(s)$, whence
    \begin{align}
        h^{\tau}_{ij}(s) \leq \frac{G_t G_r \bar{\sigma}^RCS \lambda^2 }{ (4\pi)^3 g^4(\kappa, R)} \labelrel{\leq}{(a)} \frac{G_t G_r \bar{\sigma}^RCS \lambda^2}{(4\pi)^3 g^2(\kappa, R)},
    \end{align}
    where (a) holds since $g(\kappa, R) \geq 1$ by Assumption \ref{assum:network_coverage}.
    \\ \\
    \textbf{Case 4:} $d_{ij}(s) \geq \max \{ R_i(s), R_j(s) \}$. The triangle inequality gives $R_i(s) + R_j(s) \geq d_{ij}(s)$, whence $R_i(s) \geq g(\kappa, R) - R_j(s)$. If $R_j \geq \tfrac{1}{2} g(\kappa, R)$, then
    \begin{align}
        h^{\tau}_{ij}(s) &= \frac{G_t G_r \sigma^{RCS}(s) \lambda^2}{(4\pi)^3 R^2_i(s) R^2_j(s)} \leq \frac{G_t G_r \sigma^{RCS}(s) \lambda^2}{(4\pi)^3 R^2_j(s)} \\
        &\leq \frac{G_t G_r \bar{\sigma}^{RCS} \lambda^2}{(4\pi)^3 \tfrac{1}{4} g^2(\kappa, R)}.
    \end{align}
    On the other hand, if $R_j(s) < \tfrac{1}{2} g(\kappa, R)$, then $R_i(s) \geq g(\kappa, R) - R_j(s) > \tfrac{1}{2} g(\kappa, R)$. Thus, by the same argument as above,
    \begin{align}
        h^{\tau}_{ji}(s) \leq \frac{G_t G_r \bar{\sigma}^{RCS} \lambda^2}{(4\pi)^3 \tfrac{1}{4} g^2(\kappa, R)}.
    \end{align}
    Taking the maximum over all four cases, we obtain
    \begin{align}
        h^{\tau}_{ji}(s) \leq \frac{G_t G_r \bar{\sigma}^{RCS} \lambda^2}{4^2 \pi^3 g^2(\kappa, R)}. \label{eqn:thm1:10}
    \end{align}
    
    Combining \eqref{eqn:thm1:8}, \eqref{eqn:thm1:9}, and \eqref{eqn:thm1:10}, we have
    \small
    \begin{align}
        \abs*{\alpha - \beta} \leq \bar{c} a^{max} \abs*{ \mc{N}_{\kappa}^{-1}(i) } \left[ \frac{G_t' G_r' \lambda^2}{(4 \pi)^2 g^2(\kappa, R)} + \frac{G_t G_r \bar{\sigma}^{RCS} \lambda^2}{4^2 \pi^3 g^2(\kappa, R)} \right], \label{eqn:thm1:11}
    \end{align}
    \normalsize
    which, combined with \eqref{eqn:thm1:7}, gives us
    \begin{align}
        \abs*{G^i_{\kappa}(s, a)} &\leq \frac{G_t G_r \bar{\sigma}^{RCS} \lambda^2 (a^{max})^2 \bar{c}}{(4\pi)^3} \cdot \nonumber \\
        &\hspace{5mm} \left[ \frac{G_t' G_r' \lambda^2}{(4 \pi)^2 } + \frac{G_t G_r \bar{\sigma}^{RCS} \lambda^2}{4^2 \pi^3 } \right] \frac{ \abs*{ \mc{N}_{\kappa}^{-1}(i) } }{ g^2(\kappa, R) }. \label{eqn:thm1:12}
    \end{align}
    Finally, combining \eqref{eqn:thm1:5} with \eqref{eqn:thm1:12}, we have
    \begin{align}
        \abs*{Q^{r^i}_{\theta}(s, a) - Q^{r^i}_{\theta}(\bar{s}, \bar{s}) } \leq \frac{ M \abs*{ \mc{N}_{\kappa}^{-1}(i) } }{ g^2(\kappa, R) },
    \end{align}
    where
    \small
    \begin{align}
        M = \frac{2 m G_t G_r \bar{\sigma}^{RCS} \lambda^4 (a^{max})^2 \bar{c}}{(1 - \rho)(4\pi)^3} \left[ \frac{G_t' G_r'}{(4 \pi)^2 } + \frac{G_t G_r \bar{\sigma}^{RCS}}{4^2 \pi^3 } \right]. \label{eqn:thm1:13}
    \end{align}
    \normalsize
\end{proof}


    In Theorem \ref{thm:radar_exp_decay}, the scalar $M$ intuitively corresponds to the maximum possible expected contribution of radar $j$ to $SINR_i$ before accounting for signal attenuation due to distance. When $\kappa = 1$ and $g(\kappa, R) = R$, for example, $M | \mc{N}_{1}^{-1}(i) | / R^2$ bounds the maximum possible contribution to $SINR_i$ originating outside radar $i$'s immediate communication neighborhood $\mc{N}_{1}(i)$, decayed by the square of the communication radius $R$. Furthermore, for general $g$, as $\kappa > 1$ increases $\mc{N}_{\kappa}^{-1}(i)$ will decrease and $g(\kappa, R)$ will increase, resulting in a tighter overall bound $M | \mc{N}_{\kappa}^{-1}(i) | / g^2(\kappa, R)$. This bound can be further tightened by increasing the communication radius $R$.

We now proceed to the main result of this section, which establishes bounds on the accuracy of gradient estimators constructed using only local neighborhood information. We first define local $Q$ function approximations. Fix $i \in \mc{N}$, and let $w^i : \mc{S}^{\mc{N}^{-1}_{\kappa}(i)} \times \mc{A}^{\mc{N}^{-1}_{\kappa}(i)} \rightarrow [0, 1]$ be an arbitrary weighting function satisfying
$\sum_{ \bar{s} \in \mc{S}^{\mc{N}^{-1}_{\kappa}(i)}, \bar{a} \in \mc{A}^{\mc{N}^{-1}_{\kappa}(i)} } w^i(\bar{s}, \bar{a}) = 1$.
Let $\widetilde{Q}_{\theta}^{f^i}( s^{\mc{N}_{\kappa}(i)}, a^{\mc{N}_{\kappa}(i)} )$ denote agent $i$'s local approximation of $Q^{f^i}(s, a)$, for $f^i \in \{ r^i, c^i \}$, defined by
\small
\begin{align}
    \widetilde{Q}_{\theta}^{f^i}&( s^{\mc{N}_{\kappa}(i)}, a^{\mc{N}_{\kappa}(i)} ) = \nonumber \\
    & \hspace{-8mm} \sum_{ \bar{s} \in \mc{S}^{\mc{N}^{-1}_{\kappa}(i)}, \bar{a} \in \mc{A}^{\mc{N}^{-1}_{\kappa}(i)} } \hspace{-8mm} Q_{\theta}^{f^i} \left( (s^{\mc{N}_{\kappa}(i)}, \bar{s}), (a^{\mc{N}_{\kappa}(i)}, \bar{a}) \right) w^i(\bar{s}, \bar{a}). \label{eqn:local_Q}
\end{align}
\normalsize
For any such weighting function $w^i$, the following approximation result is implied by Theorem \ref{thm:radar_exp_decay}.
\begin{theorem} \label{thm:radar_local_grads}
Let the conditions of Theorem \ref{thm:radar_exp_decay} and Assumption \ref{assum:lipschitz_score_functions} hold. Fix $i, j \in \mc{N}, f^i \in \{ r^i, c^i \}$, define $f(s, a) = \sum_{i \in \mc{N}} f^i(s^i, a^i)$, and let $M$ be as in Theorem \ref{thm:radar_exp_decay}. We have:
\begin{enumerate}[(i)]
    \small
    \item $| \widetilde{Q}_{\theta}^{f^i}( s^{\mc{N}_{\kappa}(i)}, a^{\mc{N}_{\kappa}(i)} ) - Q_{\theta}^{f^i}(s, a) | \leq \frac{M | \mc{N}_{\kappa}^{-1}(i) |}{g^2(\kappa, R)}$, \\ for all $s \in \mc{S}, a \in \mc{A}$; \label{thm:2:i}
    \item $\norm{ \widehat{h_{f^j}^i}(\theta) - \nabla_{\theta^i} J_{f^j}(\theta) } \leq \frac{M L^i | \mc{N}_{\kappa}^{-1}(j) |}{g^2(\kappa, R)}$,
    where $\widehat{h^i_{f^j}}(\theta) = \mb{E}_{\pi_{\theta}} \left[ \widetilde{Q}_{\theta}^{f^j}( s^{\mc{N}_{\kappa}(j)}, a^{\mc{N}_{\kappa}(j)} ) \nabla_{\theta^i} \log \pi^i_{\theta^i}(a^i | s^i) \right] $; \label{thm:2:ii}
    \item If Assumption \ref{assum:bdd_interagent_gradients} also holds, then $\norm{ \widehat{h^i_f}(\theta) - \nabla_{\theta^i} J_f(\theta) } \leq \frac{M \bar{n} L^i | \mc{N}_{\kappa}^{-1}(i) |}{g^2(\kappa, R)} + \abs*{ \mc{N}_{\kappa}^{-1}(i) } \varepsilon_{\kappa},$
    where $\widehat{h^i_f}(\theta) = \mb{E}_{\pi_{\theta}} \left[ \sum_{j \in \mc{N}_{\kappa}(i)} \widetilde{Q}_{\theta}^{f^j}( s^{\mc{N}_{\kappa}(j)}, a^{\mc{N}_{\kappa}(j)} ) \nabla_{\theta^i} \log \pi^i_{\theta^i}(a^i | s^i) \right]$ and $\bar{n} = \max_{j \in \mc{N}} \abs*{ \mc{N}_{\kappa}^{-1}(j)}$. \label{thm:2:iii}
    \item If Assumption \ref{assum:bdd_interagent_gradients} also holds, then $\norm{ \sum_{j \in \mc{N}_{\kappa}(i)} \eta^j \widehat{ h^i_{f^j} }(\theta) - \sum_{l \in \mc{N}} \eta^l \nabla_{\theta^i} J_{f^l}(\theta) }$ \\ \qquad $\leq \sum_{j \in \mc{N}_{\kappa}(i)} \abs*{ \eta^j } \frac{ M L^i \abs*{ \mc{N}_{\kappa}^{-1}(j) }}{ g^2(\kappa, R) } + \sum_{j \in \mc{N}_{\kappa}^{-1}(i)} \abs*{ \eta^j } \varepsilon_{\kappa}$, for all $\eta \in \mb{R}^n$. \label{thm:2:iv}
\end{enumerate}
\end{theorem}
\begin{proof}
    As in Theorem \ref{thm:radar_exp_decay}, the $f^i = c^i$ case is trivial by Assumption \ref{assum:pairwise_cost_independence}, so consider only the $f^i = r^i$ case. Fix $w^i, s \in \mc{S}$, and $a \in \mc{A}$.
    \\ \\
    \textbf{Part (i).} By the definition of $\widetilde{Q}^{r^i}_{\theta}$, we have
    \footnotesize
    \begin{align}
        &\abs*{ \widetilde{Q}^{r^i}_{\theta}( s^{\mc{N}_{\kappa}(i)}, a^{\mc{N}_{\kappa}(i)} ) - Q^{r^i}_{\theta}(s, a) } \\
        %
        %
        &\leq \hspace{-5mm} \sum_{\bar{s} \in \mc{S}^{\mc{N}_{\kappa}^{-1}(i)}, \bar{a} \in \mc{A}^{\mc{N}_{\kappa}^{-1}(i)}} \hspace{-5mm} w^i(\bar{s}, \bar{a}) \abs*{ Q^{r^i}_{\theta}\left( ( s^{\mc{N}_{\kappa}(i)}, \bar{s} ), ( a^{\mc{N}_{\kappa}(i)}, \bar{a} ) \right) - Q^{r^i}_{\theta}(s, a) } \\
        &\leq \hspace{-5mm} \sum_{\bar{s} \in \mc{S}^{\mc{N}_{\kappa}^{-1}(i)}, \bar{a} \in \mc{A}^{\mc{N}_{\kappa}^{-1}(i)}} \hspace{-10mm} w^i(\bar{s}, \bar{a}) \frac{ M \abs*{ \mc{N}_{\kappa}^{-1}(i) } }{ g^2(\kappa, R) } = \frac{ M \abs*{ \mc{N}_{\kappa}^{-1}(i) } }{ g^2(\kappa, R) }.
    \end{align}
    \normalsize
    \\ \\
    \textbf{Part (ii).} By the policy gradient theorem \cite{sutton1999policy},
    \begin{align}
        \nabla_{\theta^i} J_{r^j}(\theta) = \E_{\pi_{\theta}} \left[ Q^{r^j}_{\theta}(s, a) \nabla_{\theta^i} \log \pi^i_{\theta^i}(a^i | s^i) \right].
    \end{align}
    This means we can write
    \small
    \begin{align}
        &\norm{ \widehat{h^i_{r^j}}(\theta) - \nabla_{\theta^i} J_{r^j}(\theta) } \\
        &= \E_{\pi_{\theta}} \big[ \big( \widetilde{Q}^{r^j}_{\theta}( s^{\mc{N}_{\kappa}(j)}, a^{\mc{N}_{\kappa}(j)} ) - Q^{r^j}_{\theta}(s, a) \big) \nabla_{\theta^i} \log \pi^i_{\theta^i}(a^i | s^i) \big] \\
        &\labelrel{\leq}{(a)} \E_{\pi_{\theta}} \left[ \abs*{ \widetilde{Q}^{r^j}_{\theta}( s^{\mc{N}_{\kappa}(j)}, a^{\mc{N}_{\kappa}(j)} ) - Q^{r^j}_{\theta}(s, a) } \cdot \norm{ \nabla_{\theta^i} \log \pi^i_{\theta^i}(a^i | s^i) } \right] \\
        &\labelrel{\leq}{(b)} \E_{\pi_{\theta}} \left[ \frac{ M \abs*{ \mc {N}_{\kappa}^{-1}(j) } }{ g^2(\kappa, R) } \cdot L^i \right] = \frac{ M L^i \abs*{ \mc {N}_{\kappa}^{-1}(j) } }{ g^2(\kappa, R) },
    \end{align}
    \normalsize
    where (a) follows by Jensen's inequality and the Cauchy-Schwarz inequality, and (b) follows by Theorem \ref{thm:radar_exp_decay} and Assumption \ref{assum:lipschitz_score_functions}.
    \\ \\
    \textbf{Part (iii).} By \eqref{eqn:global_obs} and the definition of $\widehat{h^i_r}(\theta)$, we have
    \begin{align}
        \nabla_{\theta^i} J_r(\theta) &= \sum_{j \in \mc{N}} \nabla_{\theta^i} J_{r^j}(\theta),
        \quad
        \widehat{h^i_r}(\theta) = \sum_{j \in \mc{N}_{\kappa}(i)} \widehat{h^i_{r^j}}(\theta).
    \end{align}
    We thus have
    \small
    \begin{align}
        &\norm{ \widehat{h^i_r}(\theta) - \nabla_{\theta^i} J_r(\theta) } \\
        &= \norm{ \sum_{j \in \mc{N}_{\kappa}(i)} \left( \widehat{h^i_{r^j}}(\theta) - \nabla_{\theta^i} J_{r^j}(\theta) \right) - \sum_{j \in \mc{N}_{\kappa}^{-1}(i)} \nabla_{\theta^i} J_{r^j}(\theta) } \\
        &\leq \sum_{j \in \mc{N}_{\kappa}(i)} \norm{ \widehat{h^i_{r^j}}(\theta) - \nabla_{\theta^i} J_{r^j}(\theta) } + \sum_{j \in \mc{N}_{\kappa}^{-1}(i)} \norm{ \nabla_{\theta^i} J_{r^j}(\theta) } \\
        &\labelrel{\leq}{(a)} \sum_{j \in \mc{N}_{\kappa}(i)} \frac{ M L^i \abs*{ \mc{N}_{\kappa}^{-1}(j) } }{ g^2(\kappa, R) } + \sum_{j \in \mc{N}_{\kappa}^{-1}} \varepsilon_{\kappa} \\
        &\labelrel{\leq}{(b)} \frac{ M L^i \bar{n} \abs*{ \mc{N}_{\kappa}^{-1}(j) } }{ g^2(\kappa, R) } + \abs*{ \mc{N}_{\kappa}^{-1}(i) } \varepsilon_{\kappa},
    \end{align}
    \normalsize
    where (a) follows by Part (ii) and Assumption \ref{assum:bdd_interagent_gradients}, and (b) follows by the definition of $\bar{n}$ in the statement of the theorem.
    \\ \\
    \textbf{Part (iv).} Fix $\eta \in \mathbb{R}^n$. We know that
    \small
    \begin{align}
        &\norm{ \sum_{j \in \mc{N}_{\kappa}(i)} \eta^j \widehat{h^i_{r^j}}(\theta) - \sum_{j \in \mc{N}} \eta^j \nabla_{\theta^i} J_{r^j}(\theta) } \\
        &= \norm{ \sum_{j \in \mc{N}_{\kappa}(i)} \eta^j \left[ \widehat{h^i_{r^j}}(\theta) - \nabla_{\theta^i} J_{r^j}(\theta) \right] - \sum_{j \in \mc{N}_{\kappa}^{-1}(i)} \eta^j \nabla_{\theta^i} J_{r^j}(\theta) } \\
        &\labelrel{\leq}{(a)} \sum_{j \in \mc{N}_{\kappa}(i)} \abs*{\eta^j} \norm{ \widehat{h^i_{r^j}}(\theta) - \nabla_{\theta^i} J_{r^j}(\theta) } + \sum_{j \in \mc{N}_{\kappa}^{-1}(i)} \abs*{\eta^j} \norm{ \nabla_{\theta^i} J_{r^j}(\theta) } \\
        &\labelrel{\leq}{(b)} \sum_{j \in \mc{N}_{\kappa}(i)} \abs*{\eta^j} \frac{ M L^i \abs*{ \mc{N}_{\kappa}^{-1}(j) }}{g^2(\kappa, R)} + \sum_{j \in \mc{N}_{\kappa}^{-1}(i)} \abs*{\eta^j} \varepsilon_{\kappa},
    \end{align}
    \normalsize
    where (a) follows by the triangle and Cauchy-Schwarz inequalities and (b) follows by Part (ii) and Assumption \ref{assum:bdd_interagent_gradients}.
\end{proof}


Theorem \ref{thm:radar_local_grads} provides approximate policy gradient expressions that can be computed using only local neighborhood information and establishes corresponding error bounds depending on the system model, communication radius, choice of $\kappa$, policy design, and placement of radars within the network. The bounds arise due to signal decay inherent in the radar range equations \eqref{eqn:radar_range_1}-\eqref{eqn:radar_range_2}, manifested in the reward \eqref{eqn:SINR_dynamic}. These bounds provide a guide to the selection of $\kappa$, design of $g$ from Assumption \ref{assum:network_coverage}, radar placement based on the effective communication range $R$, properties of the system model captured by $M$, properties of the parametric policy class through the Lipschitz constants $L_i$, and the number of radars $n$. Importantly, the result demonstrates that, for appropriate choice of design parameters given the underlying system, the expressions presented in Theorem \ref{thm:radar_local_grads} provide good approximations of the desired policy gradients while using only local neighborhood information.
\noindent Armed with these policy gradient expressions, we next turn to development of decentralized MARL algorithms leveraging them to solve the problems proposed in Section \ref{subsec:cmamdp}.
\section{Algorithms} \label{sec:algorithms}

In this section, we derive decentralized, policy gradient-based MARL methods for solving problems \eqref{opt:Q_kappa} and \eqref{opt:R_kappa} proposed in Section \ref{subsec:cmamdp}. In each case, we solve the problem using a decentralized, stochastic policy gradient descent-ascent procedure, or decentralized saddle point policy gradient (D-SP-PG), on the Lagrangian relaxation corresponding to the original problem. Performing gradient ascent-descent on the Lagrangian relaxation is a standard solution technique for constrained RL problems \cite{paternain2019constrained}, while our extension to the MARL setting is inspired by decentralized approaches to optimization over networks \cite{koppel2015saddle, koppel2017proximity}. For each of problems \eqref{opt:Q_kappa} and \eqref{opt:R_kappa}, we first formulate the Lagrangian of the corresponding problem, then provide gradient expressions and local approximations for each agent, and finally state the corresponding decentralized algorithm.

\begin{algorithm}[t]
	\caption{SINR Maximization with Cost Constraints}
	\label{alg:sinr_max}
	\begin{algorithmic}[1]
		\STATE \textbf{Input:} stepsizes $\alpha_t, \beta_t, \zeta_t$
		\STATE \textbf{Initialize:} initialize $s_0, \theta_0, \nu_0$, set $t = \widehat{\mu^{c^i}_0} = \widehat{\mu^{r^i}_0} = 0$, and set all entries of $\widetilde{Q}^{r^i}, \widetilde{Q}^{c^i}$ to 0, for all $i \in \mc{N}$
		\FOR{agent $i \in \mc{N}$}
            \STATE share $s^i_t$ with $\mc{N}_{\kappa}(i)$, receive $s^{\mc{N}_{\kappa}(i)}_t$ from $\mc{N}_{\kappa}(i)$
            \STATE take action $a^i_t \sim \pi^i_{\theta^i_t}(\cdot | s^{\mc{N}_{\kappa}(i)}_t)$
            \STATE observe $r^i_t = r^i(s_t, a_t)$, $c^i_t = c^i(s_t, a_t)$
            %
            %
            \STATE $\widehat{ \mu^{c^i}_t } = (1 - \zeta_t) \widehat{ \mu^{c^i}_{t-1} } + \zeta_t c^i_t$
            \STATE $\widehat{ \mu^{r^i}_t } = (1 - \zeta_t) \widehat{ \mu^{r^i}_{t-1} } + \zeta_t r^i_t$
            \STATE share $a^i_t$ with $\mc{N}_{\kappa}(i)$, receive $a^{\mc{N}_{\kappa}(i)}_t$ from $\mc{N}_{\kappa}(i)$
            \STATE $\widetilde{Q}^{c^i}_t = \text{\texttt{UPDATE\_Q}}( \widetilde{Q}^{c^i}_{t-1}, c^i_t, \widehat{ \mu^{c^i}_t }, s^{\mc{N}_{\kappa}(i)}_t, a_t^{\mc{N}_{\kappa}(i)}, \zeta_t)$
            \STATE $\widetilde{Q}^{r^i}_t = \text{\texttt{UPDATE\_Q}}( \widetilde{Q}^{r^i}_{t-1}, r^i_t, \widehat{ \mu^{r^i}_t }, s^{\mc{N}_{\kappa}(i)}_t, a_t^{\mc{N}_{\kappa}(i)}, \zeta_t)$
            \STATE share $\widetilde{Q}^{r^i}_t (s_t^{\mc{N}_{\kappa}(i)}, a_t^{\mc{N}_{\kappa}(i)}), \widehat{\mu^{c^i}_t}, \nu^i_t$ with $\mc{N}_{\kappa}(i)$, receive $\widetilde{Q}^{r^j}_t (s_t^{\mc{N}_{\kappa}(j)}, a_t^{\mc{N}_{\kappa}(j)}), \widehat{\mu^{c^j}_t}, \nu^j_t$ from $\mc{N}_{\kappa}(i)$
            \STATE form estimates:
                \small
                \begin{align*}
                    \widehat{h^i_{r, t}} &= \sum_{j \in \mc{N}_{\kappa}(i) }\widetilde{Q}^{r^j}_t (s_t^{\mc{N}_{\kappa}(j)}, a_t^{\mc{N}_{\kappa}(j)}) \nabla_{\theta^i} \log \pi^i_{\theta^i_t}(a^i_t | s_t^{\mc{N}_{\kappa}(i)}) \\
                    \widehat{h^i_{c^i, t}} &= \widetilde{Q}^{c^i}_t (s_t^{\mc{N}_{\kappa}(i)}, a_t^{\mc{N}_{\kappa}(i)})\nabla_{\theta^i} \log \pi^i_{\theta^i}(a^i_t | s_t^{\mc{N}_{\kappa}(i)})
                \end{align*}
                \normalsize
            \STATE update:
                \begin{align*}
                    \theta^i_{t+1} &= \theta^i_t + \alpha_t \big( \widehat{h^i_{r, t}} - \sum_{j \in \mc{N}_{\kappa}(i)} \nu^j \widehat{h^i_{c^i, t}} \big) \\
                    \nu^i_{t+1} &= \nu^i_t - \beta_t \big( u^i - \sum_{j \in \mc{N}_{\kappa}(i)} \widehat{\mu^{c^j}_t} \big)
                \end{align*}
            \STATE $t \gets t+1$
		\ENDFOR
	\end{algorithmic}
\end{algorithm}
\begin{algorithm}[t]
	\caption{Cost Minimization with SINR Constraints}
	\label{alg:cost_min}
	\begin{algorithmic}[1]
		\STATE \textbf{Input:} stepsizes $\alpha_t, \beta_t, \zeta_t$
		\STATE \textbf{Initialize:} initialize $s_0, \theta_0, \eta_0, \nu_0$, set $t = \widehat{\mu^{c^i}_0} = \widehat{\mu^{r^i}_0} = 0$, and set all entries of $\widetilde{Q}^{r^i}, \widetilde{Q}^{c^i}$ to 0, for all $i \in \mc{N}$
		\FOR{agent $i \in \mc{N}$}
            %
            %
            %
            \STATE share $s^i_t$ with $\mc{N}_{\kappa}(i)$, receive $s^{\mc{N}_{\kappa}(i)}_t$ from $\mc{N}_{\kappa}(i)$
            \STATE take action $a^i_t \sim \pi^i_{\theta^i_t}(\cdot | s^{\mc{N}_{\kappa}(i)}_t)$
            \STATE observe $r^i_t = r^i(s_t, a_t)$, $c^i_t = c^i(s_t, a_t)$
            \STATE $\widehat{ \mu^{c^i}_t } = (1 - \zeta_t) \widehat{ \mu^{c^i}_{t-1} } + \zeta_t c^i_t$
            \STATE $\widehat{ \mu^{r^i}_t } = (1 - \zeta_t) \widehat{ \mu^{r^i}_{t-1} } + \zeta_t r^i_t$
            \STATE share $a^i_t$ with $\mc{N}_{\kappa}(i)$, receive $a^{\mc{N}_{\kappa}(i)}_t$ from $\mc{N}_{\kappa}(i)$
            %
            %
            %
            \STATE $\widetilde{Q}^{c^i}_t = \text{\texttt{UPDATE\_Q}}( \widetilde{Q}^{c^i}_{t-1}, c^i_t, \widehat{ \mu^{c^i}_t }, s^{\mc{N}_{\kappa}(i)}_t, a_t^{\mc{N}_{\kappa}(i)}, \zeta_t)$
            \STATE $\widetilde{Q}^{r^i}_t = \text{\texttt{UPDATE\_Q}}( \widetilde{Q}^{r^i}_{t-1}, r^i_t, \widehat{ \mu^{r^i}_t }, s^{\mc{N}_{\kappa}(i)}_t, a_t^{\mc{N}_{\kappa}(i)}, \zeta_t)$
            \STATE share $\widetilde{Q}^{r^i}_t (s_t^{\mc{N}_{\kappa}(i)}, a_t^{\mc{N}_{\kappa}(i)}), \widehat{\mu^{c^i}_t}, \nu^i_t$ with $\mc{N}_{\kappa}(i)$, receive $\widetilde{Q}^{r^j}_t (s_t^{\mc{N}_{\kappa}(j)}, a_t^{\mc{N}_{\kappa}(j)}), \widehat{\mu^{c^j}_t}, \nu^j_t$ from $\mc{N}_{\kappa}(i)$
            \STATE form estimates:
                \small
                \begin{align*}
                    \widehat{h^i_{c^i, t}} &= \widetilde{Q}^{c^i}_t (s_t^{\mc{N}_{\kappa}(i)}, a_t^{\mc{N}_{\kappa}(i)}) \nabla_{\theta^i} \log \pi^i_{\theta^i}(a^i_t | s_t^{\mc{N}_{\kappa}(i)}) \\
                    \widehat{h^i_{r^j, t}} &= \widetilde{Q}^{r^j}_t (s_t^{\mc{N}_{\kappa}(i)}, a_t^{\mc{N}_{\kappa}(j)}) \nabla_{\theta^i} \log \pi^i_{\theta^i}(a^i_t | s_t^{\mc{N}_{\kappa}(i)}), \\
                    &\qquad \text{for all } j \in \mc{N}_{\kappa}(i)
                \end{align*}
                \normalsize
            \STATE update:
                \small
                \begin{align*}
                    \theta^i_{t+1} &= \theta^i_t - \alpha_t \Big( \Big( 1 + \sum_{j \in \mc{N}_{\kappa}(i)} \nu^j \Big) \widehat{h^i_{c^i, t}} - \sum_{j \in \mc{N}_{\kappa}(i)} \eta^j \widehat{h^i_{r^j, t}} \Big) \\
                    \eta^i_{t+1} &= \eta^i_t + \beta_t \Big( \gamma_{min} - \widehat{\mu^{r^i}_t} \Big) \\
                    \nu^i_{t+1} &= \nu^i_t + \delta_t \Big( u^i - \sum_{j \in \mc{N}_{\kappa}(i)} \widehat{\mu^{c^j}_t} \Big)
                \end{align*}
                \normalsize
            \STATE $t \gets t+1$
		\ENDFOR
	\end{algorithmic}
\end{algorithm}
\begin{algorithm}
	\caption{\texttt{UPDATE\_Q}}
	\label{alg:update_q}
	\begin{algorithmic}[1]
		\STATE \textbf{Input:} $\widetilde{Q}^{f^i}_{t-1}, f^i_t, \widehat{ \mu^{f^i}_t }, s_t^{\mc{N}_{\kappa}(i)}, a_t^{\mc{N}_{\kappa}(i)}, \zeta_t$
        \STATE perform updates
            \small
            \begin{align*}
                \widetilde{Q}^{f^i}_t (s_t^{\mc{N}_{\kappa}(i)}, a_t^{\mc{N}_{\kappa}(i) } ) &= (1 - \zeta_t) \widetilde{Q}^{f^i}_{t-1}(s_t^{\mc{N}_{\kappa}(i)}, a_t^{\mc{N}_{\kappa}(i) } ) \\
                &+ \zeta_t \left( f^i_t - \widehat{ \mu^{f^i}_t } + \widetilde{Q}^{f^i}_{t-1}(s_t^{\mc{N}_{\kappa}(i)}, a_t^{\mc{N}_{\kappa}(i) } ) \right) \\
                \widetilde{Q}^{f^i}_t (s^{\mc{N}_{\kappa}(i)}, a^{\mc{N}_{\kappa}(i) } ) &= \widetilde{Q}^{f^i}_{t-1} (s^{\mc{N}_{\kappa}(i)}, a^{\mc{N}_{\kappa}(i) } ), \text{ for all } \\
                & (s^{\mc{N}_{\kappa}(i)}, a^{\mc{N}_{\kappa}(i) } ) \neq (s_t^{\mc{N}_{\kappa}(i)}, a_t^{\mc{N}_{\kappa}(i) } )
            \end{align*}
            \normalsize
        \RETURN $\widetilde{Q}^{f^i}_t$
	\end{algorithmic}
\end{algorithm}

\subsection{Sum-of-SINRs Maximization} 
The Lagrangian of problem \eqref{opt:Q_kappa} is given by
\begin{equation}
    \ms{L}_Q(\theta, \nu) = J_r(\theta) + \sum_{j \in \mc{N}} \nu^j \big( u^j - \sum_{k \in \mc{N}_{\kappa}(j)} J_{c^k}(\theta) \big), \label{eqn:L_Q}
\end{equation}
where we recall that $J_r(\theta) = \sum_{j \in \mc{N}} J_{r^j}(\theta)$. In order to solve \eqref{opt:Q_kappa}, our goal is to instead solve the problem $\max_{\theta} \min_{\nu} \ms{L}_Q(\theta, \nu)$ by alternating between gradient ascent in $\theta$ and gradient descent in $\nu$ using stochastic approximates of the gradient expressions
\begin{align}
    \nabla_{\theta^i} \ms{L}_Q(\theta, \nu) &= \nabla_{\theta^i} J_r(\theta) - \sum_{j \in \mc{N}} \nu^j \sum_{k \in \mc{N}_{\kappa}(j)} \nabla_{\theta^i} J_{c^k}(\theta) \label{eqn:L_Q_grad_1} \\
    &= \nabla_{\theta^i} J_r(\theta) - \sum_{j \in \mc{N}_{\kappa}(i)} \nu^j \nabla_{\theta^i} J_{c^i}(\theta) \label{eqn:L_Q_grad_2} \\
    \nabla_{\nu^i} \ms{L}_Q(\theta, \nu) &= u^i - \sum_{j \in \mc{N}_{\kappa}(i)} J_{c^j}(\theta). \label{eqn:L_Q_grad_3}
\end{align}
Here equation \eqref{eqn:L_Q_grad_2} follows from Assumption \ref{assum:pairwise_cost_independence} combined with our assumption that the $\kappa$-hop neighbor relation is symmetric, i.e., that $i \in \mc{N}_{\kappa}(j)$ if and only if $j \in \mc{N}_{\kappa}(i)$.

%
We can now present a decentralized learning algorithm for approximately solving this problem, where each agent only needs information available within its local $\kappa$-hop neighborhood. Using the approximations provided by Theorem \ref{thm:radar_local_grads}, for suitable choice of $\kappa$ and $R$ we have
\begin{align}
    \nabla_{\theta^i} \ms{L}_Q(\theta, \nu) &\approx \widehat{ h^i_r }(\theta) - \sum_{j \in \mc{N}_{\kappa}(i) } \nu^j \widehat{ h^i_{c^i} }(\theta), \\
    \nabla_{\nu^i} \ms{L}_Q(\theta, \nu) &\approx u^i - \sum_{j \in \mc{N}_{\kappa}(i) } \widehat{ \mu^{c^j} }(\theta),
\end{align}
where $\widehat{\mu^{c^j}}(\theta) \approx J_{c^j}(\theta)$ is a suitable approximation, such as a cumulative or exponential moving average. Using these expressions, we provide the D-SP-PG scheme for solving \eqref{opt:Q_kappa} in Algorithm \ref{alg:sinr_max}.
\subsection{Power Minimization with SINR Threshold}
The Lagrangian of problem \eqref{opt:R_kappa} is given by
\begin{align}
    \ms{L}_R(\theta, \eta, \nu) &= J_c(\theta) + \sum_{j \in \mc{N}} \eta^j \left( \gamma_{min} - J_{r^j}(\theta) \right) \label{eqn:R_kappa_Lagrangian_1} \\
    &\qquad - \sum_{j \in \mc{N}} \nu^j \big( u^j - \sum_{k \in \mc{N}_{\kappa}(j)} J_{c^k}(\theta) \big), \label{eqn:R_kappa_Lagrangian_2}
\end{align}
where we recall that $J_c(\theta) = \sum_{j \in \mc{N}} J_{c^j}(\theta)$. In order to solve \eqref{opt:R_kappa}, we instead solve the saddle point problem $\min_{\theta} \max_{\eta, \nu} \ms{L}_R(\theta, \eta, \nu)$ by alternating between stochastic gradient descent in $\theta$ and ascent in $\eta$ and $\nu$. Differentiating \eqref{eqn:R_kappa_Lagrangian_1}, \eqref{eqn:R_kappa_Lagrangian_2} with respect to $\theta^i$ gives
\begin{align}
    \nabla_{\theta^i} & \ms{L}_R (\theta, \eta, \nu) = \nabla_{\theta^i} J_c(\theta) - \sum_{j \in \mc{N}} \eta^j \nabla_{\theta^i} J_{r^j}(\theta) \\
    &\qquad\qquad\qquad + \sum_{j \in \mc{N}} \nu^j \sum_{k \in \mc{N}_{\kappa}(j)} \nabla_{\theta^i} J_{c^k}(\theta) \\
    &= \nabla_{\theta^i} J_{c^i}(\theta) - \sum_{j \in \mc{N}} \eta^j \nabla_{\theta^i} J_{r^j}(\theta) + \sum_{k \in \mc{N}_{\kappa}(i)} \nu^k \nabla_{\theta^i} J_{c^i}(\theta) \\
    &= \Big( 1 + \sum_{k \in \mc{N}_{\kappa}(i)} \nu^j \Big) \nabla_{\theta^i} J_{c^i}(\theta) - \sum_{j \in \mc{N}} \eta^j \nabla_{\theta^i} J_{r^j}(\theta),
\end{align}
where the first equality holds by Assumption \ref{assum:pairwise_cost_independence} combined with our assumption that the $\kappa$-hop neighbor relation is symmetric, i.e., that $i \in \mc{N}_{\kappa}(j)$ if and only if $j \in \mc{N}_{\kappa}(i)$. Finally, differentiating with respect to $\eta^i$ and $\nu^i$ yields
\begin{align}
    \nabla_{\eta^i} \ms{L}_R(\theta, \eta, \nu) &= \gamma_{min} - J_{r^i}(\theta), \\
    \nabla_{\nu^i} \ms{L}_R(\theta, \eta, \nu) &= u^i - \sum_{j \in \mc{N}_{\kappa}(i)} J_{c^j}(\theta).
\end{align}

Our goal is again to perform the necessary updates in a decentralized manner, with each agent only using locally available information. Using the approximations provided by Theorem \ref{thm:radar_local_grads}, for suitable choice of $\kappa$ and $R$ we have
\begin{align}
    \nabla_{\theta^i} \ms{L}_R(\theta, \eta, \nu) &\approx \Big( 1 + \hspace{-2mm} \sum_{k \in \mc{N}_{\kappa}(i)} \hspace{-2mm} \nu^k \Big)\widehat{h^i_{c^i}}(\theta) -  \hspace{-2mm} \sum_{j \in \mc{N}_{\kappa}(i)} \eta^j \widehat{h^i_{r^j}}(\theta) \\
    \nabla_{\eta^i} \ms{L}_R(\theta, \eta, \nu) &\approx \gamma_{min} - \widehat{\mu^{r^i}}(\theta) \\
    \nabla_{\nu^i} \ms{L}_R(\theta, \eta, \nu) &\approx u^i - \sum_{j \in \mc{N}_{\kappa}(i)} \widehat{\mu^{c^j}}(\theta),
\end{align}
where $\widehat{\mu^{r^i}}(\theta) \approx J_{r^l}(\theta) $ and $\widehat{\mu^{c^j}}(\theta) \approx J_{c^j}(\theta)$ are suitable approximations.  Using these expressions, we provide the D-SP-PG scheme for solving \eqref{opt:R_kappa} in Algorithm \ref{alg:cost_min}.

\begin{remark}
    We note that, in order to lay the groundwork for convergence analysis of our algorithms in future work, we have presented the $Q$ updates of Algorithm \ref{alg:update_q} in the easily analyzable tabular form, which is only applicable to finite state and action spaces. We emphasize that practical variants leveraging standard neural network architectures for $Q$ function approximation can be substituted in Algorithm \ref{alg:update_q} to enable application to the continuous spaces of Section \ref{sec:formulation}.
\end{remark}

\section{Conclusion} \label{sec:conclusion}

In this work, we have considered the specific use-case of power allocation for target detection in a radar network to illustrate how signal strength decay properties inherent in wireless communications and radar networks can be used to develop decentralized, scalable MARL methods.
%
%
Future directions include convergence analysis of Algorithms \ref{alg:sinr_max} and \ref{alg:cost_min}, experimental evaluation of neural network-based versions of our algorithms, extension of our approach to additional applications in communications and radar networks, and extension of our approach to joint motion planning and power allocation in target detection in mobile radar networks.


\bibliographystyle{IEEEbib}
\bibliography{refs}




\end{document}